\PassOptionsToPackage{noend}{algorithm2e}
\documentclass{opt2020} 
\pdfoutput=1
\usepackage[utf8]{inputenc} 
\usepackage[T1]{fontenc}    
\usepackage{hyperref}       
\usepackage{booktabs}       
\usepackage{nicefrac}       
\usepackage{microtype}      
\usepackage{epsf}
\usepackage{mathtools}
\usepackage{multirow}
\usepackage{multicol}
\usepackage{xspace}
\usepackage{algorithm,algorithmicx,algpseudocode}
\usepackage{listings}
\usepackage{xcolor}

\usepackage{tikz}
\usetikzlibrary{calc}
\usepackage{zref-savepos}

\usepackage{enumitem}

\usepackage{caption}
\title[Local AdaAlter]{Local AdaAlter: Communication-Efficient Stochastic Gradient Descent with Adaptive Learning Rates}

\optauthor{\Name{Cong Xie $^{\hspace{0.05in}1,2}$} \thanks{The work was done when Cong Xie was a (part-time) intern in Amazon Web Services.} \Email{cx2@illinois.edu}\\
  \Name{Oluwasanmi Koyejo $^1$} \Email{sanmi@illinois.edu}\\
  \Name{Indranil Gupta $^1$} \Email{indy@illinois.edu}\\
  \Name{Haibin Lin $^2$} \Email{haibin.lin.aws@gmail.com} \\
  \addr $^1$ Department of Computer Science, University of Illinois Urbana-Champaign \quad $^2$ Amazon Web Services
  }

\begin{document}

\allowdisplaybreaks

\maketitle

\def\Blue{\color{blue}}
\def\Purple{\color{purple}}

\def\A{{\bf A}}
\def\a{{\bf a}}
\def\B{{\bf B}}
\def\C{{\bf C}}
\def\c{{\bf c}}
\def\D{{\bf D}}
\def\d{{\bf d}}
\def\E{{\mathbb{E}}}
\def\F{{\bf F}}
\def\e{{\bf e}}
\def\f{{\bf f}}
\def\G{{\bf G}}
\def\H{{\bf H}}
\def\I{{\bf I}}
\def\K{{\bf K}}
\def\L{{\bf L}}
\def\M{{\bf M}}
\def\m{{\bf m}}
\def\N{{\bf N}}
\def\n{{\bf n}}
\def\Q{{\bf Q}}
\def\q{{\bf q}}
\def\R{{\mathbb{R}}}
\def\S{{\bf S}}
\def\s{{\bf s}}
\def\T{{\bf T}}
\def\U{{\bf U}}
\def\u{{\bf u}}
\def\V{{\bf V}}
\def\tv{{\tilde v}}
\def\v{{\bf v}}
\def\W{{\bf W}}
\def\w{{\bf w}}
\def\X{{\bf X}}
\def\x{{\bf x}}
\def\bx{{\bar{x}}}
\def\Y{{\bf Y}}
\def\y{{\bf y}}
\def\Z{{\bf Z}}
\def\tdr{{\tilde r}}
\def\z{{\bf z}}
\def\0{{\bf 0}}
\def\1{{\bf 1}}

\def\AM{{\mathcal A}}
\def\CM{{\mathcal C}}
\def\DM{{\mathcal D}}
\def\GM{{\mathcal G}}
\def\FM{{\mathcal F}}
\def\IM{{\mathcal I}}
\def\NM{{\mathcal N}}
\def\OM{{\mathcal O}}
\def\SM{{\mathcal S}}
\def\TM{{\mathcal T}}
\def\UM{{\mathcal U}}
\def\XM{{\mathcal X}}
\def\YM{{\mathcal Y}}
\def\RB{{\mathbb R}}

\def\TX{\tilde{\bf X}}
\def\tx{\tilde{\bf x}}
\def\ty{\tilde{\bf y}}
\def\TZ{\tilde{\bf Z}}
\def\tz{\tilde{\bf z}}
\def\hd{\hat{d}}
\def\HD{\hat{\bf D}}
\def\hx{\hat{\bf x}}
\def\TD{\tilde{\Delta}}
\def\tg{\tilde{g}}
\def\tmu{\tilde{\mu}}

\def\alp{\mbox{\boldmath$\alpha$\unboldmath}}
\def\bet{\mbox{\boldmath$\beta$\unboldmath}}
\def\epsi{\mbox{\boldmath$\epsilon$\unboldmath}}
\def\etab{\mbox{\boldmath$\eta$\unboldmath}}
\def\ph{\mbox{\boldmath$\phi$\unboldmath}}
\def\pii{\mbox{\boldmath$\pi$\unboldmath}}
\def\Ph{\mbox{\boldmath$\Phi$\unboldmath}}
\def\Ps{\mbox{\boldmath$\Psi$\unboldmath}}
\def\tha{\mbox{\boldmath$\theta$\unboldmath}}
\def\Tha{\mbox{\boldmath$\Theta$\unboldmath}}
\def\muu{\mbox{\boldmath$\mu$\unboldmath}}
\def\Si{\mbox{\boldmath$\Sigma$\unboldmath}}
\def\si{\mbox{\boldmath$\sigma$\unboldmath}}
\def\Gam{\mbox{\boldmath$\Gamma$\unboldmath}}
\def\Lam{\mbox{\boldmath$\Lambda$\unboldmath}}
\def\De{\mbox{\boldmath$\Delta$\unboldmath}}
\def\Ome{\mbox{\boldmath$\Omega$\unboldmath}}
\def\TOme{\mbox{\boldmath$\hat{\Omega}$\unboldmath}}
\def\vps{\mbox{\boldmath$\varepsilon$\unboldmath}}
\newcommand{\ti}[1]{\tilde{#1}}
\def\Ncal{\mathcal{N}}
\def\argmax{\mathop{\rm argmax}}
\def\argmin{\mathop{\rm argmin}}
\providecommand{\abs}[1]{\lvert#1\rvert}
\providecommand{\norm}[2]{\lVert#1\rVert_{#2}}

\def\Zs{{\Z_{\mathrm{S}}}}
\def\Zl{{\Z_{\mathrm{L}}}}
\def\Yr{{\Y_{\mathrm{R}}}}
\def\Yg{{\Y_{\mathrm{G}}}}
\def\Yb{{\Y_{\mathrm{B}}}}
\def\Ar{{\A_{\mathrm{R}}}}
\def\Ag{{\A_{\mathrm{G}}}}
\def\Ab{{\A_{\mathrm{B}}}}
\def\As{{\A_{\mathrm{S}}}}
\def\Asr{{\A_{\mathrm{S}_{\mathrm{R}}}}}
\def\Asg{{\A_{\mathrm{S}_{\mathrm{G}}}}}
\def\Asb{{\A_{\mathrm{S}_{\mathrm{B}}}}}
\def\Or{{\Ome_{\mathrm{R}}}}
\def\Og{{\Ome_{\mathrm{G}}}}
\def\Ob{{\Ome_{\mathrm{B}}}}

\def\Expect{\mathbb{E}}

\def\vec{\mathrm{vec}}
\def\fold{\mathrm{fold}}
\def\index{\mathrm{index}}
\def\sgn{\mathrm{sgn}}
\def\tr{\mathrm{tr}}
\def\rk{\mathrm{rank}}
\def\diag{\mathsf{diag}}
\def\const{\mathrm{Const}}
\def\dg{\mathsf{dg}}
\def\st{\mathsf{s.t.}}
\def\vect{\mathsf{vec}}
\def\MCAR{\mathrm{MCAR}}
\def\MSAR{\mathrm{MSAR}}
\def\etal{{\em et al.\/}\,}
\def\prox{\mathrm{prox}^h_\gamma}
\newcommand{\indep}{{\;\bot\!\!\!\!\!\!\bot\;}}

\newcommand{\mtrxt}[1]{{#1}^\top}
\newcommand{\mtrx}[4]{\left[\begin{matrix}#1 & #2 \\ #3 & #4\end{matrix}\right]}
\DeclarePairedDelimiter\vnorm{\lVert}{\rVert}
\DeclarePairedDelimiterX{\innerprod}[2]{\langle}{\rangle}{#1, #2}

\def\Lsize{\hbox{\space \raise-2mm\hbox{$\textstyle \L \atop \scriptstyle {m\times 3n}$} \space}}
\def\Ssize{\hbox{\space \raise-2mm\hbox{$\textstyle \S \atop \scriptstyle {m\times 3n}$} \space}}
\def\Osize{\hbox{\space \raise-2mm\hbox{$\textstyle \Ome \atop \scriptstyle {m\times 3n}$} \space}}
\def\Tsize{\hbox{\space \raise-2mm\hbox{$\textstyle \T \atop \scriptstyle {3n\times n}$} \space}}
\def\Bsize{\hbox{\space \raise-2mm\hbox{$\textstyle \B \atop \scriptstyle {m\times n}$} \space}}

\newcommand{\twopartdef}[4]
{
	\left\{
		\begin{array}{ll}
			#1 & \mbox{if } #2 \\
			#3 & \mbox{if } #4
		\end{array}
	\right.
}

\newcommand{\tabincell}[2]{\begin{tabular}{@{}#1@{}}#2\end{tabular}}

\renewcommand{\algorithmicrequire}{\textbf{Input:}} 
\renewcommand{\algorithmicensure}{\textbf{Output:}} 
\newcommand{\NewProcedure}[1]{\Statex\hspace{-\algorithmicindent} {\large\underline{\textbf{{#1}:}}} \setcounter{ALG@line}{0}}

\DeclarePairedDelimiter\ceil{\lceil}{\rceil}
\DeclarePairedDelimiter\floor{\lfloor}{\rfloor}

\newcommand{\ip}[2]{\left\langle #1, #2 \right \rangle}

\newtheorem{assumption}{Assumption}

\newcommand{\comment}[1]{%
 \tag*{$\triangleright$ #1}
}

\newcommand{\mcir}[1]{{\mbox{\large \textcircled{\small #1}}}}

\newcounter{NoTableEntry}
\renewcommand*{\theNoTableEntry}{NTE-\the\value{NoTableEntry}}

\newcommand*{\strike}[2]{%
  \multicolumn{1}{#1}{%
    \stepcounter{NoTableEntry}%
    \vadjust pre{\zsavepos{\theNoTableEntry t}}
    \vadjust{\zsavepos{\theNoTableEntry b}}
    \zsavepos{\theNoTableEntry l}
    \hspace{0pt plus 1filll}%
    #2
    \hspace{0pt plus 1filll}%
    \zsavepos{\theNoTableEntry r}
    \tikz[overlay]{%
      \draw
        let
          \n{llx}={\zposx{\theNoTableEntry l}sp-\zposx{\theNoTableEntry r}sp-\tabcolsep},
          \n{urx}={\tabcolsep},
          \n{lly}={\zposy{\theNoTableEntry b}sp-\zposy{\theNoTableEntry r}sp},
          \n{ury}={\zposy{\theNoTableEntry t}sp-\zposy{\theNoTableEntry r}sp}
        in
        (\n{llx}, \n{lly}) -- (\n{urx}, \n{ury})
      ;
    }%
  }%
}





\begin{abstract}%
When scaling distributed training, the communication overhead is often the bottleneck. In this paper, we propose a novel SGD variant with reduced communication and adaptive learning rates. We prove the convergence of the proposed algorithm for smooth but non-convex problems. Empirical results show that the proposed algorithm significantly reduces the communication overhead, which, in turn, reduces the training time by up to 30\% for the 1B word dataset.
\end{abstract}


\section{Introduction}

Stochastic Gradient Descent~(SGD) and its variants are commonly used for training deep neural networks. We can distribute the workload across multiple workers, which results in distributed SGD with data parallelism~\citep{Goyal2017AccurateLM,You2017ScalingSB,You2017ImageNetTI,You2019LargeBO}. 
A larger number of workers accelerates the training, but also increases the overall communication cost. In the worst case, it saturates the network interconnections.
In this paper, we reduce the communication overhead by skipping communication rounds, and periodically averaging the models across the workers.
Such an approach is called local SGD~\citep{Stich2018LocalSC,Lin2018DontUL,Yu2018ParallelRS,Wang2018CooperativeSA,Yu2019OnTL}. There are other approaches to reduce the communication overhead of distributed SGD, such as quantization~\citep{Seide20141bitSG,Strom2015ScalableDD,Wen2017TernGradTG,Alistarh2016QSGDCS,Bernstein2018signSGDCO,Karimireddy2019ErrorFF,Zheng2019CommunicationEfficientDB} and sparsification~\citep{Aji2017SparseCF,Stich2018SparsifiedSW,Jiang2018ALS,xie2020cser}. 

Adaptive learning rate methods adapt coordinate-wise dynamic learning rates by accumulating the historical gradients. Examples include  AdaGrad~\citep{McMahan2010AdaptiveBO,duchi2011adaptive}, RMSProp~\citep{tieleman2012lecture}, AdaDelta~\citep{Zeiler2012ADADELTAAA}, and Adam~\citep{Kingma2014AdamAM}. Along similar lines, recent research has shown that AdaGrad can converge without explicitly decreasing the learning rate~\citep{ward2019adagrad,zou2019sufficient}. We note that these methods were not designed for local SGD.
Nevertheless, in distributed SGD, it remains unclear how to use infrequent synchronization to reduce the communication overhead in SGD with adaptive learning rates. In this paper, we answer this question by introducing staleness to the updates of the adaptive learning rates. To be more specific, the update of the adaptive variables is delayed until the communication round.

We propose a novel SGD variant based on AdaGrad, and adopt the concept of local SGD to reduce the communication. 
To the best of our knowledge, this paper is the first to theoretically and empirically study local SGD with adaptive learning rates. 
The main contributions are as follows:

\setitemize[0]{leftmargin=*}
\begin{itemize}[topsep=2pt,itemsep=0pt,partopsep=0pt,parsep=0pt]
\item We propose {\it Local AdaAlter}, a new technique to lazily update the adaptive variables. This enables communication reduction via periodic synchronization for SGD with {adaptive learning rates}.
\item We prove the convergence of the proposed algorithm for non-convex problems.
\item We show empirically that Local AdaAlter significantly reduces the communication overhead, thus also cutting training time by up to 30\% for the 1B word. 
\end{itemize}

\section{Related work}

In this paper, we consider a centralized server-worker architecture, also known as the Parameter Server~(PS)~\citep{li2014scaling,li2014communication,ho2013more,Peng2019AGC}. 
A common alternative is the AllReduce algorithm~\citep{Sergeev2018HorovodFA,walker1996mpi}. Most of the existing deep-learning frameworks, such as Tensorflow~\citep{Abadi2016TensorFlowAS}, and PyTorch~\citep{Steiner2019PyTorchAI} support either of them.
%
Similar to local SGD, there are other SGD variants that also reduce the communication overhead by skipping synchronization rounds, such as federated learning~\citep{konevcny2016federated,mcmahan2016communication} and EASGD~\citep{Zhang2014DeepLW}.
%
%
In this paper, we focus on synchronous training with homogeneous workers. In contrast, asynchronous training~\citep{Zinkevich2009SlowLA,Niu2011HOGWILDAL,zhao2016fast} is faster when there are stragglers, but noisier due to asynchrony~\citep{Dutta2018SlowAS}. 

\section{Problem formulation}

We consider the optimization problem: 
$
    \min_{x \in \R^d} F(x), 
$
where $F(x) = \frac{1}{n} \sum_{i \in [n]} \E_{z_i \sim \mathcal{D}_i} f(x; z_i)$, for $\forall i \in [n]$, $z_i$ is sampled from the local dataset $\mathcal{D}_i$ on the $i$th worker.
We solve this problem in a distributed manner with $n$ workers. 
In each iteration, the $i$th worker will sample a mini-batch of independent samples from the dataset $\mathcal{D}_i$, and compute the stochastic gradient $G_{i} = \nabla f(x; z_{i}), \forall i \in [n]$, where $z_{i} \sim \mathcal{D}_i$. 
Note that we assume non-IID workers, i.e., $\mathcal{D}_i \neq \mathcal{D}_j, \forall i \neq j$. 

\begin{table}[htb]
\caption{Notations}
\vspace{-0.5cm}
\label{tbl:notations}
\begin{center}
\begingroup
\small
\renewcommand{\arraystretch}{0.5}
\setlength{\tabcolsep}{3pt}
\begin{tabular}{|l|l|}
\hline 
{\bf Notation}  & {\bf Description} \\ \hline
$x \in \R^d$    & Model parameter \\ \hline
$F(x)$, $F_i(x)$, $f_i(x)$     & $F(x) = \frac{1}{n} \sum_{i \in [n]} F_i(x)$; $F_i(x)=\E[ f(x; z_i) ], z_i \sim \mathcal{D}_i$; $\E[ f_i(x) ] = F_i(x)$ \\ \hline
$T, t$    & Total number and index of iterations  \\ \hline
$G_t$, $(G_t)_j$, $(\nabla F_t)_j$    & Stochastic gradient $\E[G_t] = \nabla F(x_t)$, $(\cdot)_j$ is the $j$th coordinate, $j \in [d]$ \\ \hline
$(G_{i,t})_j$    & The $j$th coordinate of $G_{i,t}$, on the $i$th worker, $i \in [n]$, $j \in [d]$ \\ \hline
$\circ$ &   Hadamard~(coordinate-wise) product \\ \hline
$B^2_t$    & $B^2_t = b_0^2 \1 + \frac{1}{n} \sum_{i \in [n]} \sum_{s=1}^t G_{i,s} \circ G_{i,s}$, $B^2_0 = b_0^2 \1$ \\ \hline
$B_t$, $\frac{1}{B_t}$    & $\left[\sqrt{\left( B^2_t \right)_1}, \ldots, \sqrt{\left( B^2_t \right)_d} \right]^\top$, $\left[\frac{1}{\sqrt{\left( B^2_t \right)_1}}, \ldots, \frac{1}{\sqrt{\left( B^2_t \right)_d}} \right]^\top$ \\ \hline
$\frac{G_t}{B_t}$, $\frac{G_t}{\sqrt{B^2_t + \epsilon^2 \1}}$    & $G_t \circ \frac{1}{B_t} = \left[\frac{(G_t)_1}{\sqrt{\left( B^2_t \right)_1}}, \ldots, \frac{(G_t)_d}{\sqrt{\left( B^2_t \right)_d}} \right]^\top$, $\left[\frac{(G_t)_1}{\sqrt{\left( B^2_t \right)_1 + \epsilon^2}}, \ldots, \frac{(G_t)_d}{\sqrt{\left( B^2_t \right)_d + \epsilon^2}} \right]^\top$ \\ \hline
\end{tabular}
\endgroup
\vspace{-0.8cm}
\end{center}
\end{table}

\section{Methodology}

First, we introduce two SGD variants that are highly related to our work: AdaGrad and local SGD. Then, we will propose a new algorithm: local AdaAlter.

\subsection{Preliminary}

To help understand our proposed algorithm, we first introduce the classic SGD variant with adaptive learning rate: AdaGrad. The detailed algorithm is shown in Algorithm~\ref{alg:dist_adagrad}. The general idea is to accumulate the gradients coordinate-wise as the denominator to normalize the gradients.

We adopt the concept of local SGD to reduce the communication overhead. The vanilla local SGD algorithm is shown in Algorithm~\ref{alg:local_sgd}. Local SGD skips the communication rounds, and synchronizes/averages the model parameters for every $H$ iterations. Thus, on average, the communication overhead is reduced by the factor of $H$, compared to fully synchronous SGD.

\noindent
\begin{minipage}[t]{0.49\textwidth}\vspace{-0.3cm}
\begin{algorithm}[H]
\caption{Distributed AdaGrad}
Initialize $x_0$, $\epsilon^2$, $B^2_0 = \0$ \\
\For{iteration $t \in [T]$}{
    \For{workers $i \in [n]$ in parallel}{
        $G_{i, t} = \nabla f(x_{t-1}; z_{i, t})$, $z_{i, t} \sim \mathcal{D}_i$ \\
        $G_{t} = \frac{1}{n} \sum_{i \in [n]} G_{i, t}$ \\
        $B^2_{t} \leftarrow B^2_{t-1} +  G_{t} \circ G_{t}$ \\
    	$x_{t} \leftarrow x_{t-1} - \eta \frac{G_{t}}{\sqrt{B^2_{t} + \epsilon^2 \1}} $ 
    }
}
\label{alg:dist_adagrad}
\end{algorithm}
\end{minipage}
\hfill
\begin{minipage}[t]{0.49\textwidth}\vspace{-0.3cm}
\begin{algorithm}[H]
\caption{Local SGD}
Initialize $x_{1, 0} = \ldots = x_{n,0} = x_0$ \\
\For{iteration $t \in [T]$}{
    \For{workers $i \in [n]$ in parallel}{
        $G_{i, t} = \nabla f(x_{i, t-1}; z_{i, t})$, $z_{i, t} \sim \mathcal{D}_i$ \\
        $y_{i, t} \leftarrow x_{i, t-1} - \eta G_{i, t} $ \\
    	\lIf{$\mod(t, H) \neq 0$}{
    	    $x_{i, t} \leftarrow y_{i, t}$ 
    	}
    	\lElse{
    	    $x_{i, t} \leftarrow \frac{1}{n} \sum_{k \in [n]} y_{k, t}$
    	}
    }
}
\label{alg:local_sgd}
\end{algorithm}
\end{minipage}

\subsection{Local AdaAlter}

We propose an SGD variant based on AdaGrad, namely, local AdaAlter, which skips synchronization rounds, and periodically averages the model parameters \emph{and} the accumulated denominators after every $H$ iterations. The detailed algorithm is shown in Algorithm~\ref{alg:local_adaalter}. Note that in the communication rounds, AdaAlter has to synchronize not only the model parameters, but also the accumulated denominators across the workers. Thus, compared to the distributed AdaGrad~(Algorithm~\ref{alg:dist_adagrad}), local AdaAlter~(Algorithm~\ref{alg:local_adaalter}) reduces the communication overhead to $\frac{2}{H}$ on average.

\vspace{-0.2cm}
\begin{algorithm}[hbt!]
\caption{Local AdaAlter}
Initialize $x_{1, 0} = \ldots = x_{n,0} = x_0$, $B^2_{1,0} = \ldots = B^2_{n,0}  = b_0^2 \1$, $\epsilon^2$ \\
\For{iteration $t \in [T]$}{
    \For{workers $i \in [n]$ in parallel}{
        $t' = \mod(t-1, H) + 1$ \\
        $G_{i, t} = \nabla f(x_{i, t-1}; z_{i, t})$, $z_{i, t} \sim \mathcal{D}_i$ \\
        $y_{i, t} \leftarrow x_{i, t-1} - \eta \frac{G_{i, t}}{\sqrt{B^2_{i, t-t'} + t' \epsilon^2 \1}} $; 
    	$A^2_{i, t} \leftarrow B^2_{i, t-1} + G_{i, t} \circ G_{i, t}$ \\
    	\lIf{$\mod(t, H) \neq 0$}{
    	    $x_{i, t} \leftarrow y_{i, t}$; 
    	    $B^2_{i, t} \leftarrow A^2_{i, t}$ 
    	}
    	\lElse{
    	    Synchronize: $x_{i, t} \leftarrow \frac{1}{n} \sum_{k \in [n]} y_{k, t}$; 
    	    $B^2_{i, t} \leftarrow \frac{1}{n} \sum_{k \in [n]} A^2_{k, t}$ 
    	} 
    }
}
\label{alg:local_adaalter}
\end{algorithm}
\vspace{-0.2cm}

\textbf{Lazy update of the denominators:} In AdaGrad, a small positive constant $\epsilon$ is added for the numerical stability, in case that the denominator $B_t^2$ is too small. However, in AdaAlter, $\epsilon^2$ acts as a placeholder for the yet-to-be-added $G_{i, t} \circ G_{i, t}$. Thus, after $t'$ local steps without synchronization, such placeholder becomes $t'\epsilon^2$. The denominators $B^2_{i,t}$ are updated in the synchronization rounds only, which guarantees that the denominators are the same on different workers in the local iterations. 
In a nutshell, in AdaAlter, the denominators $B^2_{i,t}$ are lazily updated to enable the infrequent synchronization. 
The key idea is to use $B^2_{i, t-t'} + t' \epsilon^2 \1$ as a placeholder before synchronization. 
Note that even if we take $H=1$, AdaAlter is different from AdaGrad due to the lazy update.

The lazy update of the denominators keeps the adaptive learning rates synchronized across different workers, thus mitigates the noise caused by the local steps. However, it also incurs staleness in the adaptivity. In our experiments, we show that such small staleness does not affect the accuracy.


\section{Theoretical analysis}

In this section, we prove the convergence of Algorithm~\ref{alg:local_adaalter} for smooth but non-convex problems, with constant learning rate $\eta$. First, we introduce some assumptions for our convergence analysis.
\begin{assumption} (Smoothness)
We assume that $F(x)$ and $F_i(x), \forall i \in [n]$ are $L$-smooth:
$
F_i(y) - F_i(x) 
\leq \ip{\nabla F_i(x)}{y-x} + \frac{L}{2} \|y-x\|^2, \forall x, y.
$
\label{asm:smoothness}
\end{assumption}
\begin{assumption} (Bounded gradients)
For any stochastic gradient $G_{i, t} = \nabla f_i(x_t)$, we assume bounded coordinates $(G_{i, t})^2_j \leq \rho^2, \forall j \in [d]$, or simply $\|G_{i, t}\|_{\infty} \leq \rho$.
\label{asm:gradient}
\end{assumption}





To analyze Algorithm~\ref{alg:local_adaalter}, we introduce the following auxiliary variable:
$
    \bar{x}_{t} = \frac{1}{n} \sum_{i \in [n]} x_{i, t}.
$
We show that the sequence $\bar{x}_{0}, \ldots, \bar{x}_{T}$ converges to a critical point.
The detailed proof is in Appendix A. 
\begin{theorem}~(Convergence of local AdaAlter~(Algorithm~\ref{alg:local_adaalter})) \label{thm:local_adaalter}
Taking arbitrary $\epsilon > 0$, $\eta \leq \frac{1}{L}$ in Algorithm~\ref{alg:local_adaalter}, and $b_0 \geq 1$, under Assumption~\ref{asm:smoothness} and \ref{asm:gradient}, Algorithm~\ref{alg:local_adaalter} converges to a critical point:
$
\frac{\E \left[ \sum_{t=1}^T \| \nabla F(\bar{x}_{t-1}) \|^2 \right]}{T} 
\leq \mathcal{O}\left( \frac{1}{\eta \sqrt{T}} \right) + \mathcal{O}\left( \frac{\eta^2 H^2\log (T)}{\sqrt{T}} \right) + \mathcal{O}\left( \frac{\eta \log (T)}{n \sqrt{T}} \right).
$
\end{theorem}

With the constant learning rate $\eta$, local AdaAlter converges to a critical point when $T \rightarrow +\infty$. Increasing the number of workers $n$ reduces the variance. Compared to the fully synchronous AdaAlter, local AdaAlter has the extra noise proportional to $H^2$ due to the reduced communication.

\section{Experiments}

In this section, we empirically evaluate the proposed algorithm. 

\subsection{Multi-GPU experiment on 1B Word}
AdaGrad is mostly successful on language models.
Thus, we conduct experiments on the 1B Word benchmark dataset~\citep{Chelba2013OneBW}.
We train Big LSTM model with 10\% dropout~(LSTM-2048-512~\cite{Jzefowicz2016ExploringTL}).

\subsubsection{Evaluation setup}

Our experiments are conducted on a single machine with 8 GPUs~(an AWS P3.16 instance with 8 NVIDIA V100 GPUs, with 16GB memory per GPU). The batch size is 256 per GPU. We tune the learning rates in the range of $[0.2, 0.8]$ on the training data, and report the best results. Each experiment is composed of 50 epochs. Each epoch processes $20k \times 8 \times 256$ data samples. We repeat each experiment 5 times and take the average.
In all the experiments, we take $\epsilon=1$, $b_0 = 1$. 


The typical measure used for language models is perplexity~(PPL).
We evaluate the following performance metrics to test the reduction of communication overhead and the convergence:
i) the time consumed by one epoch versus different number of GPU workers;
ii) the perplexity on the test dataset versus time;
iii) the perplexity on the test dataset versus the number of epochs.

\subsection{Practical remarks for AdaAlter on 1B Word}

There are some additional remarks for using local AdaAlter in practice.

\noindent\textbf{Warm-up Learning Rates:}
When using AdaAlter, we observe that it behaves almost the same as AdaGrad, except at the beginning, the denominator $B_t^2$ is too small for AdaAlter. Thus, we add a warm-up mechanism for AdaAlter:
$
    \eta_t \leftarrow \eta \times \min\left( 1, \frac{t}{warm\_up\_steps} \right),
$
where $warm\_up\_steps$ is a hyperparameter. In the first $warm\_up\_steps$ iterations, the learning rate will gradually increase from $\frac{\eta}{warm\_up\_steps}$ to $\eta$. In our default setting where we use 8 GPU workers with batch size $256$, we take $\eta = 0.5$ and $warm\_up\_steps = 600$.

\noindent\textbf{Scaling Learning Rates:}
The original baseline is conducted on 4 GPU with batch size $128$ per GPU, and learning rate $0.2$. When the batch size increases by $k$, it is a common strategy to re-scale the learning rate by $k$ or $\sqrt{k}$~\citep{Goyal2017AccurateLM,You2017ScalingSB,You2017ImageNetTI,You2019LargeBO}. In our experiments, we use 8 GPU with batch size $256$ per GPU. Thus, we tuned $\eta$ in the range of $[0.4, 0.8]$, and found that $\eta = 0.5$ results in the best performance.

\subsubsection{Evaluation Rsults}
Figure~\ref{fig:lstm_time_throughput} illustrates the time consumed by one epoch and the throughput with different numbers of workers and different algorithms. We vary the synchronization periods $H$ for local AdaAlter. It is shown that local AdaAlter efficiently reduces the communication overhead.

\noindent
\begin{minipage}[t]{0.495\textwidth}\vspace{-0.3cm}
\begin{figure}[H]
\centering
\includegraphics[width=\textwidth,height=3.1cm]{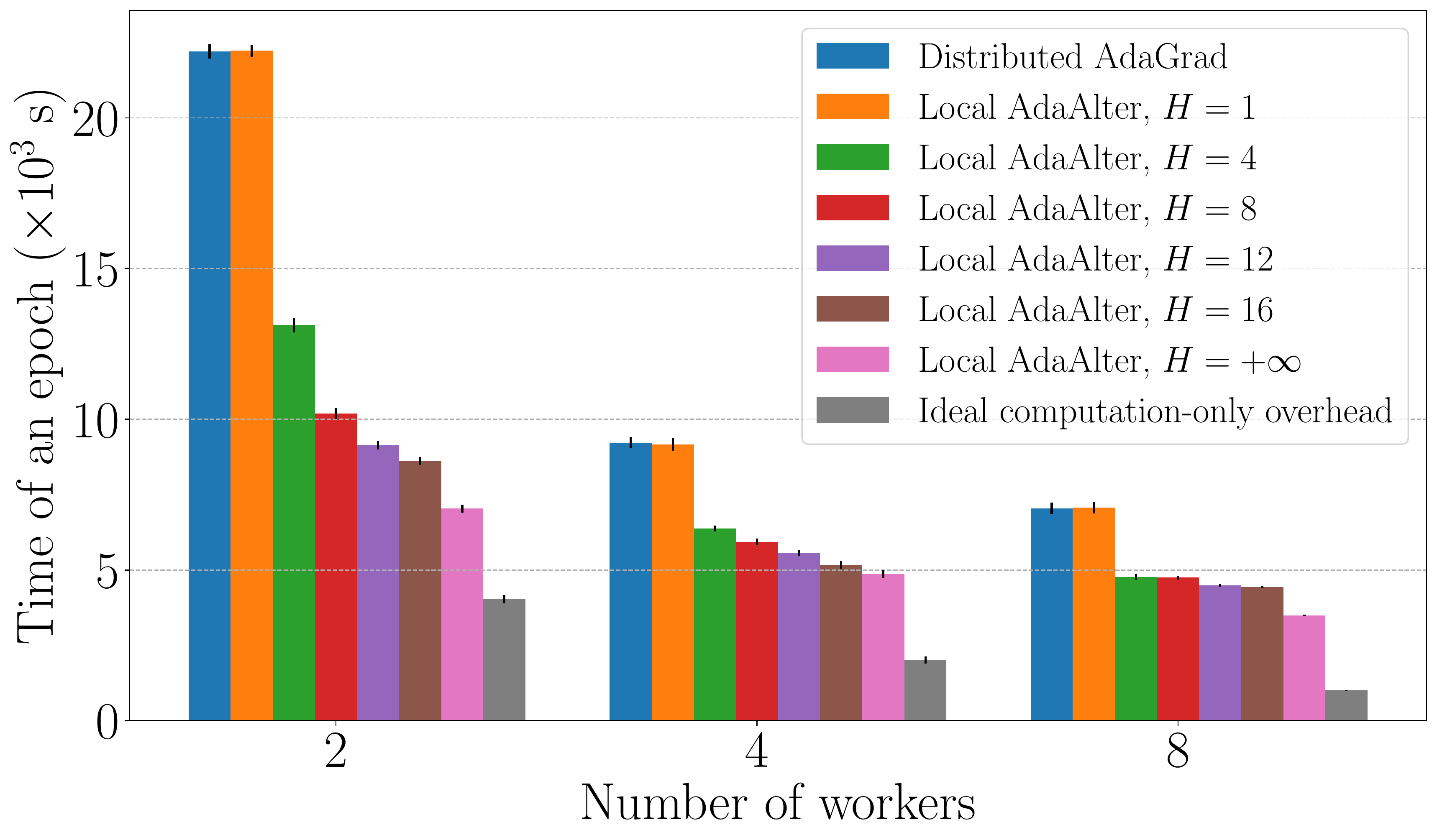}
\vspace{-0.7cm}
\caption{Time consumed by one epoch versus different numbers of workers.
}
\label{fig:lstm_time_throughput}
\end{figure}
\end{minipage}
\hfill
\begin{minipage}[t]{0.495\textwidth}\vspace{-0.3cm}
\begin{table}[H]
\caption{Test PPL and time at the end of training, for LSTM-2048-512 on 1B word dataset.}
\vspace{-0.5cm}
\label{tbl:convergence}
\begin{center}
\renewcommand{\arraystretch}{0.5}
\setlength{\tabcolsep}{3pt}
\begin{tabular}{|l|l|r|}
\hline 
{\bf Method}  & {\bf Test PPL} & {\bf Time (hours)} \\ \hline
AdaGrad & $44.58 \pm 0.02$ & 98.05 \\ \hline
Local AdaAlter & & \\ \hline
$H=1$ & $44.36 \pm 0.01$ & 98.47 \\ \hline
$H=4$ & $44.08 \pm 0.05$ & 69.17 \\ \hline
$H=8$ & $44.26 \pm 0.10$ & 67.41 \\ \hline
$H=12$ & $44.30 \pm 0.11$ & 65.49 \\ \hline
$H=16$ & $44.51 \pm 0.08$ & 64.22 \\ \hline
\end{tabular}
\end{center}
\end{table}
\end{minipage}

\begin{figure*}[htb!]
\centering
\subfigure[Test perplexity versus training time]{\includegraphics[width=0.495\textwidth,height=4cm]{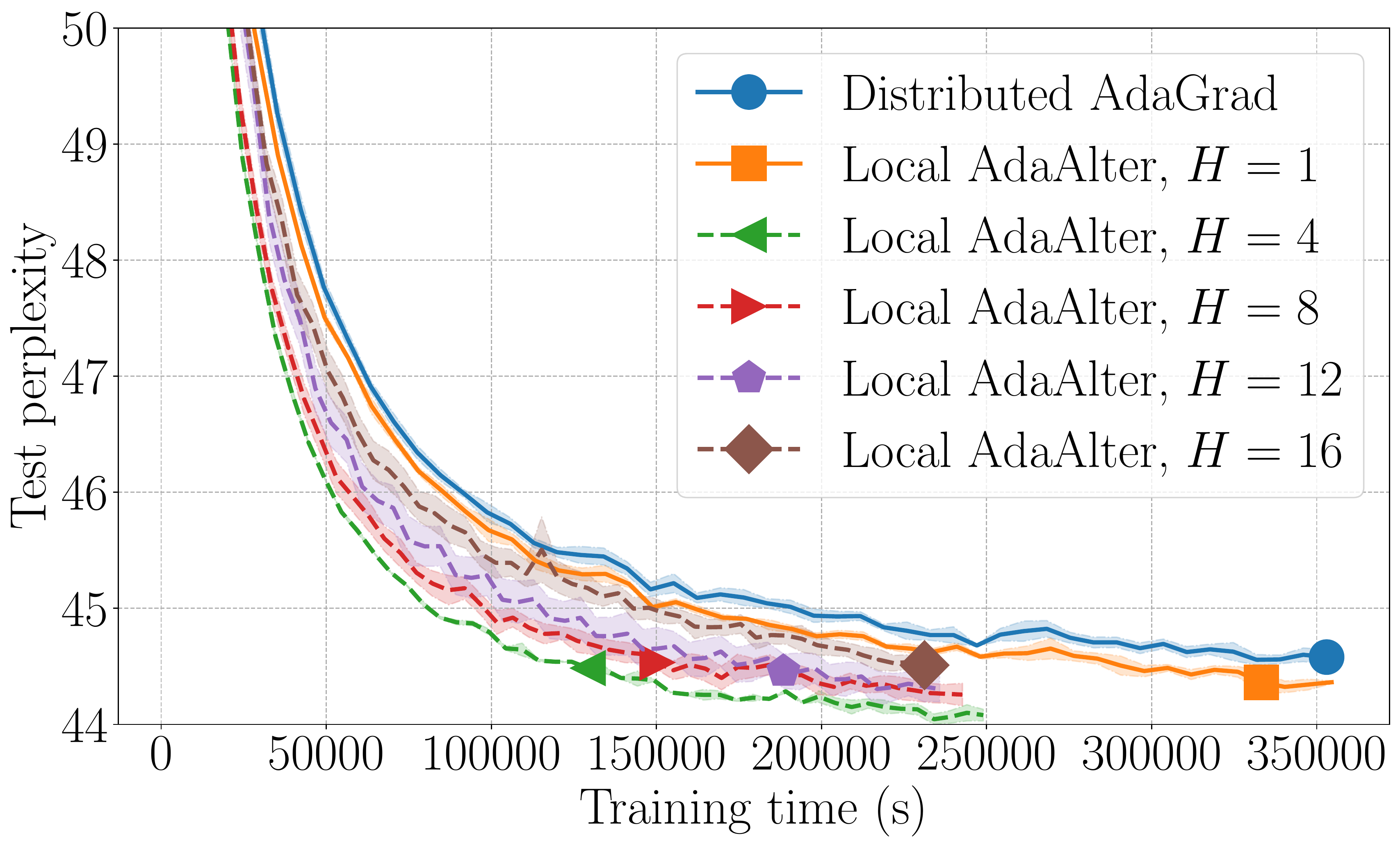}}
\subfigure[Test perplexity versus epochs]{\includegraphics[width=0.495\textwidth,height=4cm]{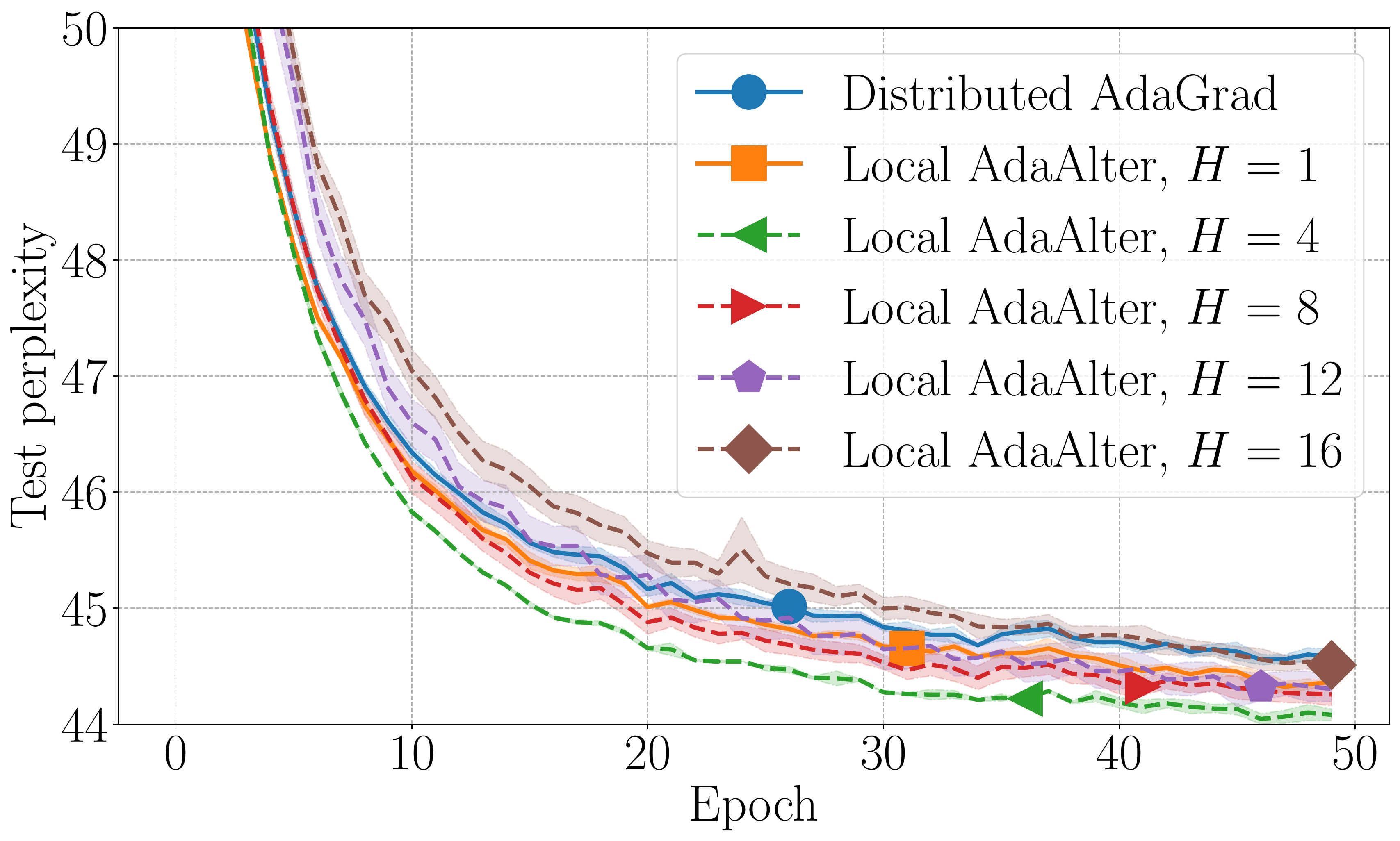}}
\subfigure[Training loss versus training time]{\includegraphics[width=0.495\textwidth,height=3.6cm]{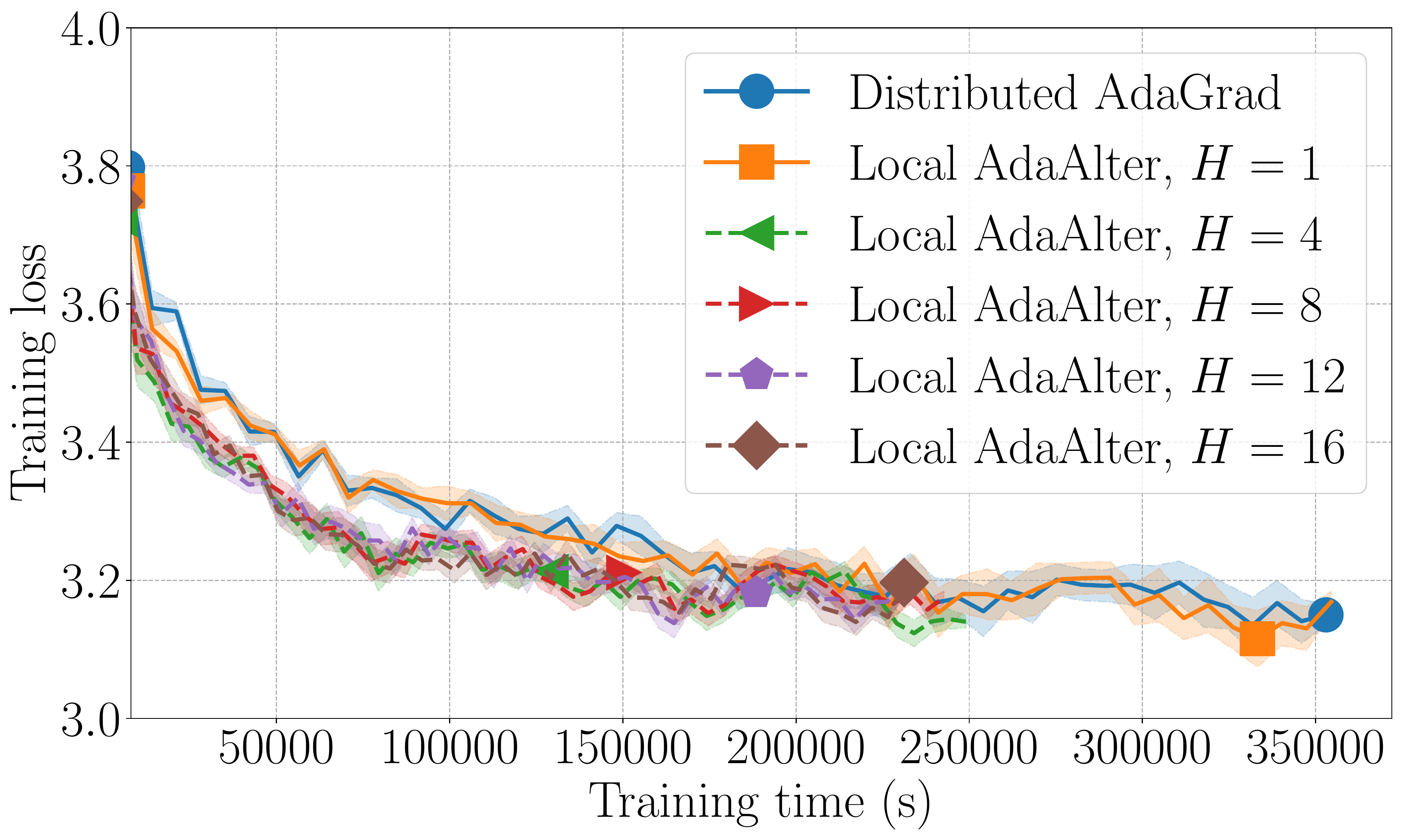}}
\subfigure[Training loss versus epochs]{\includegraphics[width=0.495\textwidth,height=3.6cm]{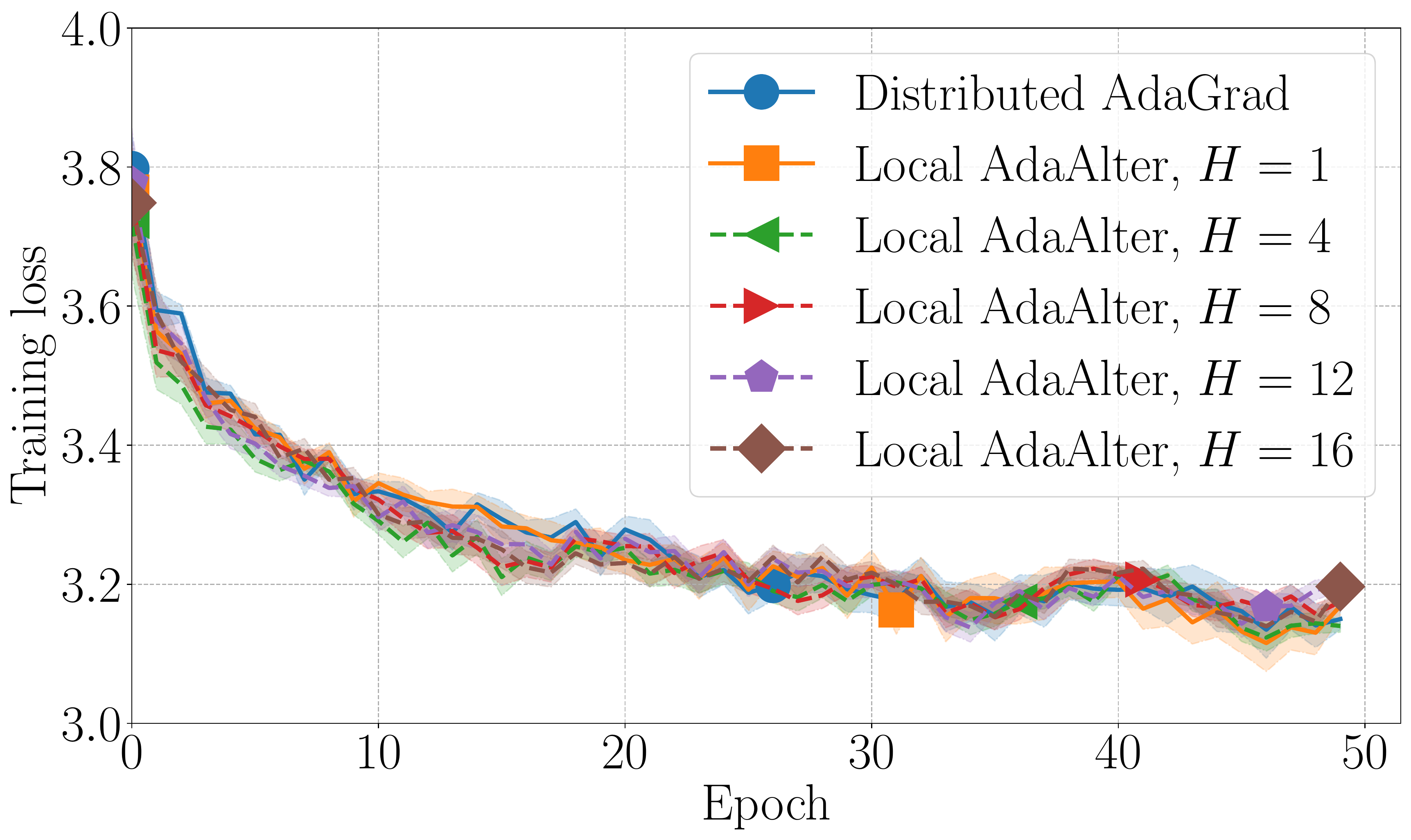}}
\caption{Evaluation of different algorithms, for LSTM-2048-512 on 1B word dataset. We use a single machine with 8 GPU workers and local batch size $256$ on each GPU.  For all experiments, we take the learning rate $\eta=0.5$. For local AdaAlter, we take $warm\_up\_steps = 600$ warm-up steps.}
\label{fig:convergence_enlarge}
\end{figure*}

The overall overhead can be decomposed into three parts: computation, communication, and data loading.
The baseline ``Local AdaAlter, $H=+\infty$'' is evaluated by manually removing the communication, and measures the ideal overhead without communication. The baseline ``Ideal computation-only overhead'' is evaluated by manually removing both the communication and the data-loading. These two baselines illustrate the ideal lower bounds of the training time.

Figure~\ref{fig:convergence_enlarge} illustrates the perplexity on the test dataset and the loss on the training dataset with different algorithms. Compared to vanilla distributed AdaGrad, local AdaAlter enjoys equivalent convergence rate, but takes much less time. To reach the same perplexity, local AdaAlter can reduce almost 30\% of the training time.

In Table~\ref{tbl:convergence}, we report the perplexity and consumed time at the end of training for different algorithms. We can see that local AdaAlter produces comparable performance to the fully synchronous AdaGrad and AdaAlter, on the test dataset, with much less training time, and acceptable variance. 

\subsection{Discussion}

We can see that the fully synchronous AdaGrad or AdaAlter~($H=1$) are very slow. Local AdaAlter reduces almost 30\% of the training time compared to the fully synchronous AdaGrad or AdaAlter.

As we expected, Figure~\ref{fig:convergence_enlarge} and Table~\ref{tbl:convergence} show that larger $H$ reduces more communication overhead, but also results in worse perplexity, which validates our theoretical analysis in Theorem~\ref{thm:local_adaalter}: when $H$ increases, the noise in the convergence also increases. Taking $H=4$ gives the best trade-off between the communication overhead and the test perplexity. 

Interestingly, as shown in Figure~\ref{fig:convergence_enlarge}(b), local AdaAlter with $H > 1$ has slightly better perplexity on the test dataset, compared to the fully synchronous AdaGrad and AdaAlter with $H=1$. Although, our theoretical analysis indicates that local AdaAlter has larger variance compared to the fully synchronous version, such conclusion only applies to the training loss. In fact, there is previous work~\citep{Lin2018DontUL} showing that local SGD potentially generalizes better than the fully synchronous SGD.
We also notice that when $H$ is too large, such benefit will be overwhelmed by the large noise.

We also observe that almost all the algorithms do not scale well when scaling from 4 to 8 workers. The major reason is that all the workers are placed in the same machine, where  CPU resources are limited. When there are too many workers, the data-loading becomes a bottleneck.
That is also the reason why different $H$ does not show much difference when using 8 GPU workers.

\section{Conclusion}

We propose a novel SGD algorithm: local AdaAlter, which reduces the communication overhead by skipping the synchronization rounds, and adopts adaptive learning rates. We show that the algorithm provably converges. Our empirical results also show accelerated training compared to baselines. 

\acks{This work was funded in part by the following grants: NSF IIS 1909577, NSF CNS 1908888, NSF CCF 1934986 and a JP Morgan Chase Fellowship, along with computational resources donated by Intel, AWS, and Microsoft Azure.}

\bibliography{adaalter_opt2020}

\begin{thebibliography}{44}
\providecommand{\natexlab}[1]{#1}
\providecommand{\url}[1]{\texttt{#1}}
\expandafter\ifx\csname urlstyle\endcsname\relax
  \providecommand{\doi}[1]{doi: #1}\else
  \providecommand{\doi}{doi: \begingroup \urlstyle{rm}\Url}\fi

\bibitem[Abadi et~al.(2016)Abadi, Barham, Chen, Chen, Davis, Dean, Devin,
  Ghemawat, Irving, Isard, Kudlur, Levenberg, Monga, Moore, Murray, Steiner,
  Tucker, Vasudevan, Warden, Wicke, Yu, and Zhang]{Abadi2016TensorFlowAS}
Mart{\'i}n Abadi, Paul Barham, Jianmin Chen, Zhifeng Chen, Andy Davis, Jeffrey
  Dean, Matthieu Devin, Sanjay Ghemawat, Geoffrey Irving, Michael Isard,
  Manjunath Kudlur, Josh Levenberg, Rajat Monga, Sherry Moore, Derek~Gordon
  Murray, Benoit Steiner, Paul~A. Tucker, Vijay Vasudevan, Pete Warden, Martin
  Wicke, Yuan Yu, and Xiaoqiang Zhang.
\newblock Tensorflow: A system for large-scale machine learning.
\newblock In \emph{OSDI}, 2016.

\bibitem[Aji and Heafield(2017)]{Aji2017SparseCF}
Alham~Fikri Aji and Kenneth Heafield.
\newblock Sparse communication for distributed gradient descent.
\newblock In \emph{EMNLP}, 2017.

\bibitem[Alistarh et~al.(2016)Alistarh, Grubic, Li, Tomioka, and
  Vojnovic]{Alistarh2016QSGDCS}
Dan Alistarh, Demjan Grubic, Jerry Li, Ryota Tomioka, and Milan Vojnovic.
\newblock Qsgd: Communication-efficient sgd via gradient quantization and
  encoding.
\newblock In \emph{NIPS}, 2016.

\bibitem[Bernstein et~al.(2018)Bernstein, Wang, Azizzadenesheli, and
  Anandkumar]{Bernstein2018signSGDCO}
Jeremy Bernstein, Yu-Xiang Wang, Kamyar Azizzadenesheli, and Anima Anandkumar.
\newblock signsgd: compressed optimisation for non-convex problems.
\newblock In \emph{ICML}, 2018.

\bibitem[Chelba et~al.(2013)Chelba, Mikolov, Schuster, Ge, Brants, and
  Koehn]{Chelba2013OneBW}
Ciprian Chelba, Tomas Mikolov, Mike Schuster, Qi~Ge, Thorsten Brants, and
  Phillipp Koehn.
\newblock One billion word benchmark for measuring progress in statistical
  language modeling.
\newblock In \emph{INTERSPEECH}, 2013.

\bibitem[Duchi et~al.(2011)Duchi, Hazan, and Singer]{duchi2011adaptive}
John Duchi, Elad Hazan, and Yoram Singer.
\newblock Adaptive subgradient methods for online learning and stochastic
  optimization.
\newblock \emph{Journal of Machine Learning Research}, 12\penalty0
  (Jul):\penalty0 2121--2159, 2011.

\bibitem[Dutta et~al.(2018)Dutta, Joshi, Ghosh, Dube, and
  Nagpurkar]{Dutta2018SlowAS}
Sanghamitra Dutta, Gauri Joshi, Soumyadip Ghosh, Parijat Dube, and Priya
  Nagpurkar.
\newblock Slow and stale gradients can win the race: Error-runtime trade-offs
  in distributed sgd.
\newblock In \emph{AISTATS}, 2018.

\bibitem[Goyal et~al.(2017)Goyal, Doll{\'a}r, Girshick, Noordhuis, Wesolowski,
  Kyrola, Tulloch, Jia, and He]{Goyal2017AccurateLM}
Priya Goyal, Piotr Doll{\'a}r, Ross~B. Girshick, Pieter Noordhuis, Lukasz
  Wesolowski, Aapo Kyrola, Andrew Tulloch, Yangqing Jia, and Kaiming He.
\newblock Accurate, large minibatch sgd: Training imagenet in 1 hour.
\newblock \emph{ArXiv}, abs/1706.02677, 2017.

\bibitem[Ho et~al.(2013)Ho, Cipar, Cui, Lee, Kim, Gibbons, Gibson, Ganger, and
  Xing]{ho2013more}
Qirong Ho, James Cipar, Henggang Cui, Seunghak Lee, Jin~Kyu Kim, Phillip~B
  Gibbons, Garth~A Gibson, Greg Ganger, and Eric~P Xing.
\newblock More effective distributed ml via a stale synchronous parallel
  parameter server.
\newblock In \emph{Advances in neural information processing systems}, pages
  1223--1231, 2013.

\bibitem[Jiang and Agrawal(2018)]{Jiang2018ALS}
Peng Jiang and Gagan Agrawal.
\newblock A linear speedup analysis of distributed deep learning with sparse
  and quantized communication.
\newblock In \emph{NeurIPS}, 2018.

\bibitem[J{\'o}zefowicz et~al.(2016)J{\'o}zefowicz, Vinyals, Schuster, Shazeer,
  and Wu]{Jzefowicz2016ExploringTL}
Rafal J{\'o}zefowicz, Oriol Vinyals, Mike Schuster, Noam Shazeer, and Yonghui
  Wu.
\newblock Exploring the limits of language modeling.
\newblock \emph{ArXiv}, abs/1602.02410, 2016.

\bibitem[Karimireddy et~al.(2019)Karimireddy, Rebjock, Stich, and
  Jaggi]{Karimireddy2019ErrorFF}
Sai~Praneeth Karimireddy, Quentin Rebjock, Sebastian~U. Stich, and Martin
  Jaggi.
\newblock Error feedback fixes signsgd and other gradient compression schemes.
\newblock In \emph{ICML}, 2019.

\bibitem[Kingma and Ba(2014)]{Kingma2014AdamAM}
Diederik~P. Kingma and Jimmy Ba.
\newblock Adam: A method for stochastic optimization.
\newblock \emph{CoRR}, abs/1412.6980, 2014.

\bibitem[Konevcn{\`y} et~al.(2016)Konevcn{\`y}, McMahan, Yu, Richt{\'a}rik,
  Suresh, and Bacon]{konevcny2016federated}
Jakub Konevcn{\`y}, H~Brendan McMahan, Felix~X Yu, Peter Richt{\'a}rik,
  Ananda~Theertha Suresh, and Dave Bacon.
\newblock Federated learning: Strategies for improving communication
  efficiency.
\newblock \emph{arXiv preprint arXiv:1610.05492}, 2016.

\bibitem[Li et~al.(2014{\natexlab{a}})Li, Andersen, Park, Smola, Ahmed,
  Josifovski, Long, Shekita, and Su]{li2014scaling}
Mu~Li, David~G Andersen, Jun~Woo Park, Alexander~J Smola, Amr Ahmed, Vanja
  Josifovski, James Long, Eugene~J Shekita, and Bor-Yiing Su.
\newblock Scaling distributed machine learning with the parameter server.
\newblock In \emph{OSDI}, volume~14, pages 583--598, 2014{\natexlab{a}}.

\bibitem[Li et~al.(2014{\natexlab{b}})Li, Andersen, Smola, and
  Yu]{li2014communication}
Mu~Li, David~G Andersen, Alexander~J Smola, and Kai Yu.
\newblock Communication efficient distributed machine learning with the
  parameter server.
\newblock In \emph{Advances in Neural Information Processing Systems}, pages
  19--27, 2014{\natexlab{b}}.

\bibitem[Lin et~al.(2018)Lin, Stich, and Jaggi]{Lin2018DontUL}
Tao Lin, Sebastian~U. Stich, and Martin Jaggi.
\newblock Don't use large mini-batches, use local sgd.
\newblock \emph{ArXiv}, abs/1808.07217, 2018.

\bibitem[McMahan and Streeter(2010)]{McMahan2010AdaptiveBO}
H.~Brendan McMahan and Matthew~J. Streeter.
\newblock Adaptive bound optimization for online convex optimization.
\newblock In \emph{COLT}, 2010.

\bibitem[McMahan et~al.(2016)McMahan, Moore, Ramage, Hampson,
  et~al.]{mcmahan2016communication}
H~Brendan McMahan, Eider Moore, Daniel Ramage, Seth Hampson, et~al.
\newblock Communication-efficient learning of deep networks from decentralized
  data.
\newblock \emph{arXiv preprint arXiv:1602.05629}, 2016.

\bibitem[Niu et~al.(2011)Niu, Recht, R{\'e}, and Wright]{Niu2011HOGWILDAL}
Feng Niu, Benjamin Recht, Christopher R{\'e}, and Stephen~J. Wright.
\newblock Hogwild!: A lock-free approach to parallelizing stochastic gradient
  descent.
\newblock In \emph{NIPS}, 2011.

\bibitem[Peng et~al.(2019)Peng, Zhu, Chen, Bao, Yi, Lan, Wu, and
  Guo]{Peng2019AGC}
Y~Peng, Y~Zhu, Y~Chen, Y~Bao, B~Yi, C~Lan, C~Wu, and C~Guo.
\newblock A generic communication scheduler for distributed dnn training
  acceleration.
\newblock In \emph{the 27th ACM Symposium on Operating Systems Principles (ACM
  SOSP 2019)}, 2019.

\bibitem[Seide et~al.(2014)Seide, Fu, Droppo, Li, and Yu]{Seide20141bitSG}
Frank Seide, Hao Fu, Jasha Droppo, Gang Li, and Dong Yu.
\newblock 1-bit stochastic gradient descent and its application to
  data-parallel distributed training of speech dnns.
\newblock In \emph{INTERSPEECH}, 2014.

\bibitem[Sergeev and Balso(2018)]{Sergeev2018HorovodFA}
Alexander Sergeev and Mike~Del Balso.
\newblock Horovod: fast and easy distributed deep learning in tensorflow.
\newblock \emph{ArXiv}, abs/1802.05799, 2018.

\bibitem[Steiner et~al.(2019)Steiner, DeVito, Chintala, Gross, Paszke, Massa,
  Lerer, Chanan, Lin, Yang, Desmaison, Tejani, Kopf, Bradbury, Antiga, Raison,
  Gimelshein, Chilamkurthy, Killeen, Fang, and Bai]{Steiner2019PyTorchAI}
Benoit Steiner, Zachary DeVito, Soumith Chintala, Sam Gross, Adam Paszke,
  Francisco Massa, Adam Lerer, Gregory Chanan, Zeming Lin, Edward Yang, Alban
  Desmaison, Alykhan Tejani, Andreas Kopf, James Bradbury, Luca Antiga, Martin
  Raison, Natalia Gimelshein, Sasank Chilamkurthy, Trevor Killeen, Lu~Fang, and
  Junjie Bai.
\newblock Pytorch: An imperative style, high-performance deep learning library.
\newblock In \emph{NeurIPS}, 2019.

\bibitem[Stich(2018)]{Stich2018LocalSC}
Sebastian~U. Stich.
\newblock Local sgd converges fast and communicates little.
\newblock \emph{ArXiv}, abs/1805.09767, 2018.

\bibitem[Stich et~al.(2018)Stich, Cordonnier, and Jaggi]{Stich2018SparsifiedSW}
Sebastian~U. Stich, Jean-Baptiste Cordonnier, and Martin Jaggi.
\newblock Sparsified sgd with memory.
\newblock \emph{ArXiv}, abs/1809.07599, 2018.

\bibitem[Strom(2015)]{Strom2015ScalableDD}
Nikko Strom.
\newblock Scalable distributed dnn training using commodity gpu cloud
  computing.
\newblock In \emph{INTERSPEECH}, 2015.

\bibitem[Tieleman and Hinton(2012)]{tieleman2012lecture}
Tijmen Tieleman and Geoffrey Hinton.
\newblock Lecture 6.5-rmsprop: Divide the gradient by a running average of its
  recent magnitude.
\newblock \emph{COURSERA: Neural networks for machine learning}, 4\penalty0
  (2):\penalty0 26--31, 2012.

\bibitem[Walker and Dongarra(1996)]{walker1996mpi}
David~W Walker and Jack~J Dongarra.
\newblock Mpi: a standard message passing interface.
\newblock \emph{Supercomputer}, 12:\penalty0 56--68, 1996.

\bibitem[Wang and Joshi(2018)]{Wang2018CooperativeSA}
Jianyu Wang and Gauri Joshi.
\newblock Cooperative sgd: A unified framework for the design and analysis of
  communication-efficient sgd algorithms.
\newblock \emph{ArXiv}, abs/1808.07576, 2018.

\bibitem[Ward et~al.(2019)Ward, Wu, and Bottou]{ward2019adagrad}
Rachel Ward, Xiaoxia Wu, and Leon Bottou.
\newblock Adagrad stepsizes: sharp convergence over nonconvex landscapes.
\newblock In \emph{International Conference on Machine Learning}, pages
  6677--6686, 2019.

\bibitem[Wen et~al.(2017)Wen, Xu, Yan, Wu, Wang, Chen, and
  Li]{Wen2017TernGradTG}
Wei Wen, Cong Xu, Feng Yan, Chunpeng Wu, Yandan Wang, Yiran Chen, and Hai Li.
\newblock Terngrad: Ternary gradients to reduce communication in distributed
  deep learning.
\newblock In \emph{NIPS}, 2017.

\bibitem[Xie et~al.(2020)Xie, Zheng, Koyejo, Gupta, Li, and Lin]{xie2020cser}
Cong Xie, Shuai Zheng, Oluwasanmi~O Koyejo, Indranil Gupta, Mu~Li, and Haibin
  Lin.
\newblock Cser: Communication-efficient sgd with error reset.
\newblock In \emph{Advances in Neural Information Processing Systems}, 2020.

\bibitem[You et~al.(2017{\natexlab{a}})You, Gitman, and
  Ginsburg]{You2017ScalingSB}
Yang You, Igor Gitman, and Boris Ginsburg.
\newblock Scaling sgd batch size to 32k for imagenet training.
\newblock \emph{ArXiv}, abs/1708.03888, 2017{\natexlab{a}}.

\bibitem[You et~al.(2017{\natexlab{b}})You, Zhang, Hsieh, Demmel, and
  Keutzer]{You2017ImageNetTI}
Yang You, Zhao Zhang, Cho-Jui Hsieh, James Demmel, and Kurt Keutzer.
\newblock Imagenet training in minutes.
\newblock In \emph{ICPP}, 2017{\natexlab{b}}.

\bibitem[You et~al.(2019)You, Li, Reddi, Hseu, Kumar, Bhojanapalli, Song,
  Demmel, and Hsieh]{You2019LargeBO}
Yang You, Jing Li, Sashank Reddi, Jonathan Hseu, Sanjiv Kumar, Srinadh
  Bhojanapalli, Xiaodan Song, James Demmel, and Cho-Jui Hsieh.
\newblock Large batch optimization for deep learning: Training bert in 76
  minutes.
\newblock \emph{arXiv preprint arXiv:1904.00962}, 2019.

\bibitem[Yu et~al.(2018)Yu, Yang, and Zhu]{Yu2018ParallelRS}
Hao Yu, Sen~Xiang Yang, and Shenghuo Zhu.
\newblock Parallel restarted sgd with faster convergence and less
  communication: Demystifying why model averaging works for deep learning.
\newblock In \emph{AAAI}, 2018.

\bibitem[Yu et~al.(2019)Yu, Jin, and Yang]{Yu2019OnTL}
Hao Yu, Rong Jin, and Sen~Xiang Yang.
\newblock On the linear speedup analysis of communication efficient momentum
  sgd for distributed non-convex optimization.
\newblock In \emph{ICML}, 2019.

\bibitem[Zeiler(2012)]{Zeiler2012ADADELTAAA}
Matthew~D. Zeiler.
\newblock Adadelta: An adaptive learning rate method.
\newblock \emph{ArXiv}, abs/1212.5701, 2012.

\bibitem[Zhang et~al.(2014)Zhang, Choromanska, and LeCun]{Zhang2014DeepLW}
Sixin Zhang, Anna Choromanska, and Yann LeCun.
\newblock Deep learning with elastic averaging sgd.
\newblock In \emph{ICLR}, 2014.

\bibitem[Zhao and Li(2016)]{zhao2016fast}
Shen-Yi Zhao and Wu-Jun Li.
\newblock Fast asynchronous parallel stochastic gradient descent: A lock-free
  approach with convergence guarantee.
\newblock In \emph{Thirtieth AAAI Conference on Artificial Intelligence}, 2016.

\bibitem[Zheng et~al.(2019)Zheng, Huang, and
  Kwok]{Zheng2019CommunicationEfficientDB}
Shuai Zheng, Ziyue Huang, and James~T. Kwok.
\newblock Communication-efficient distributed blockwise momentum sgd with
  error-feedback.
\newblock \emph{ArXiv}, abs/1905.10936, 2019.

\bibitem[Zinkevich et~al.(2009)Zinkevich, Smola, and
  Langford]{Zinkevich2009SlowLA}
Martin Zinkevich, Alexander~J. Smola, and John Langford.
\newblock Slow learners are fast.
\newblock In \emph{NIPS}, 2009.

\bibitem[Zou et~al.(2019)Zou, Shen, Jie, Zhang, and Liu]{zou2019sufficient}
Fangyu Zou, Li~Shen, Zequn Jie, Weizhong Zhang, and Wei Liu.
\newblock A sufficient condition for convergences of adam and rmsprop.
\newblock In \emph{Proceedings of the IEEE Conference on Computer Vision and
  Pattern Recognition}, pages 11127--11135, 2019.

\end{thebibliography}

\newpage
\clearpage
\onecolumn

\appendix

\vspace*{0.1cm}
\begin{center}
	\Large\textbf{Appendix}
\end{center}
\vspace*{0.1cm}

\setcounter{theorem}{0}

\section{Proofs}

\begin{lemma}~(\cite{zou2019sufficient}, Lemma 15)
For any non-negative sequence $a_0, a_1, \ldots, a_T$, we have
\begin{align*}
    \sum_{t=1}^T \frac{a_t}{a_0 + \sum_{s=1}^t a_s} \leq \log \left( a_0 + \sum_{t=1}^T a_t \right) - \log(a_0).
\end{align*}
\end{lemma}

To analyze Algorithm~\ref{alg:local_adaalter}, we introduce the following auxiliary variable:
\begin{align*}
    \bar{x}_{t} = \frac{1}{n} \sum_{i \in [n]} x_{i, t}.
\end{align*}
Also, note that in Algorithm~\ref{alg:local_adaalter}, $B_{i, t-t'}$ is synchronized. Thus, we denote
\begin{align*}
    \bar{B}_{t-t'} = B_{1, t-t'} = \ldots = B_{n, t-t'}.
\end{align*}

\begin{theorem}
Taking arbitrary $\epsilon > 0$ in Algorithm~\ref{alg:local_adaalter}, and $b_0 \geq 1$. Under Assumption~\ref{asm:smoothness} and \ref{asm:gradient}, Algorithm~\ref{alg:local_adaalter} converges to a critical point:
By telescoping and taking total expectation, we have 
\begin{align*}
&\frac{\E \left[ \sum_{t=1}^T \| \nabla F(\bar{x}_{t-1}) \|^2 \right]}{T} \\
&\leq \frac{2 \sqrt{b^2_0 + \frac{T\epsilon^2}{p^2}} \E\left[ F(\bar{x}_{t0}) - F(\bar{x}_{T}) \right]}{\eta T} 
+ \left[ 4 \eta^2 L^2 H^2 + \frac{1}{n} L \eta  \right] \frac{d \log\left( b_0^2 + T \rho^2 \right)\sqrt{b^2_0 + \frac{T\epsilon^2}{p^2}}}{T p^2} \\
&\leq \mathcal{O}\left( \frac{1}{\eta \sqrt{T}} \right) + \mathcal{O}\left( \frac{\eta^2 H^2\log (T)}{\sqrt{T}} \right) + \mathcal{O}\left( \frac{\eta \log (T)}{n \sqrt{T}} \right).
\end{align*}
\end{theorem}

\begin{proof}

For convenience, we write the stochastic gradient $\nabla f(x_{i, t-1}; z_{i, t})$ as $\nabla f_i(x_{i, t-1})$, and we have $\E[\nabla f_i(x_{i, t-1})] = \nabla F_i(x_{i, t-1})$.
Using $L$-smoothness, we have
\begin{align*}
&F(\bar{x}_t) - F(\bar{x}_{t-1}) \\
&\leq -\eta \ip{\nabla F(\bar{x}_{t-1})}{\frac{1}{n} \sum_{i \in [n]} \frac{G_{i, t}}{\sqrt{B^2_{i, t-t'} + t' \epsilon^2 \1}}} + \frac{L \eta^2}{2} \left\| \frac{1}{n} \sum_{i \in [n]} \frac{G_{i, t}}{\sqrt{B^2_{i, t-t'} + t' \epsilon^2 \1}} \right\|^2 \\
&\leq \sum_{j=1}^d \underbrace{ -\eta \frac{1}{n} \sum_{i \in [n]} \frac{ ( \nabla F(\bar{x}_{t-1}) )_j (G_{i, t})_j}{\sqrt{\left( \bar{B}_{t-t'} \right)^2_j + t' \epsilon^2}} }_{\mcir{1}}  + \frac{L \eta^2}{2} \sum_{j=1}^d \underbrace{ \frac{ (\frac{1}{n}  \sum_{i \in [n]} G_{i, t})^2_j}{\left( \bar{B}_{t-t'} \right)^2_j + t' \epsilon^2} }_{\mcir{2}}.
\end{align*}

Conditional on the previous states, taking expectation on both sides, we have
\begin{align*}
&\E\left[ \mcir{1} \right] \\
&= -\eta \frac{ \left( \nabla F(\bar{x}_{t-1}) \right)_j \left(\frac{1}{n} \sum_{i \in [n]} \nabla F_i(x_{i, t-1}) \right)_j}{\sqrt{\left( \bar{B}_{t-t'} \right)^2_j + t' \epsilon^2}} \\
&= \underbrace{ -\frac{\eta}{2} \frac{ \left( \nabla F(\bar{x}_{t-1}) \right)^2_j }{\sqrt{\left( \bar{B}_{t-t'} \right)^2_j + t' \epsilon^2}} }_{\mcir{3}} \\
&\quad \underbrace{ -\frac{\eta}{2} \frac{ \left(\frac{1}{n} \sum_{i \in [n]} \nabla F_i(x_{i, t-1}) \right)^2_j }{\sqrt{\left( \bar{B}_{t-t'} \right)^2_j + t' \epsilon^2}} }_{\mcir{4}} \\
&\quad \underbrace{ +\frac{\eta}{2} \frac{ \left( \nabla F(\bar{x}_{t-1}) - \frac{1}{n} \sum_{i \in [n]} \nabla F_i(x_{i, t-1}) \right)^2_j }{\sqrt{\left( \bar{B}_{t-t'} \right)^2_j + t' \epsilon^2}} }_{\mcir{5}}.
\end{align*}

Again, conditional on the previous states, taking expectation on both sides, we have
\begin{align*}
&\E\left[ \mcir{2} \right] \\
&= \E\left[ \frac{ \left( \frac{1}{n}  \sum_{i \in [n]} \nabla f_i(x_{i, t-1}) \right)^2_j}{\left( \bar{B}_{t-t'} \right)^2_j + t' \epsilon^2} \right] \\
&= \E\left[ \frac{ \left( \frac{1}{n}  \sum_{i \in [n]} ( \nabla f_i(x_{i, t-1}) - \nabla F_i(x_{i, t-1}) + \nabla F_i(x_{i, t-1})) \right)^2_j}{\left( \bar{B}_{t-t'} \right)^2_j + t' \epsilon^2} \right] \\
&= \underbrace{ \E\left[ \frac{ \left( \frac{1}{n}  \sum_{i \in [n]} (\nabla f_i(x_{i, t-1}) - \nabla F_i(x_{i, t-1})) \right)^2_j }{\left( \bar{B}_{t-t'} \right)^2_j + t' \epsilon^2} \right] }_{\mcir{6}} \\
&\quad + \underbrace{ \E\left[ \frac{ \left( \frac{1}{n}  \sum_{i \in [n]} \nabla F_i(x_{i, t-1}) \right)^2_j}{\left( \bar{B}_{t-t'} \right)^2_j + t' \epsilon^2} \right] }_{\mcir{7}} \\
\end{align*}

In the following steps, we bound the terms $\mcir{3}$-$\mcir{7}$, respectively.

Taking $p = \min(\frac{\epsilon}{\rho}, 1) \leq 1$, we have $\epsilon \geq p \rho$, or $\rho \leq \frac{\epsilon}{p}$. Thus, we have 
$\left( \bar{B}_{t-t'} \right)^2_j + t' \epsilon^2 
\leq b^2_0 + (t-t')\rho^2 + t' \epsilon^2 
\leq b^2_0 + (t-t')\frac{\epsilon^2}{p^2} + t' \epsilon^2
\leq b^2_0 + t\frac{\epsilon^2}{p^2}
\leq b^2_0 + T\frac{\epsilon^2}{p^2}
$
\begin{align*}
\sum_{j=1}^d \mcir{3} 
\leq \sum_{j=1}^d -\frac{\eta}{2} \frac{ \left( \nabla F(\bar{x}_{t-1}) \right)^2_j }{\sqrt{b^2_0 + T\frac{\epsilon^2}{p^2}}}
= -\frac{\eta}{2} \frac{ \| \nabla F(\bar{x}_{t-1}) \|^2 }{\sqrt{b^2_0 + T\frac{\epsilon^2}{p^2}}}.
\end{align*}

Since $\left( \bar{B}_{t-t'} \right)^2_j + t' \epsilon^2 \geq b^2_0 \geq 1$, taking $\eta \leq \frac{1}{L}$, we have
\begin{align*}
&\mcir{4} + \frac{L\eta^2}{2} \mcir{7} \\
&= -\frac{\eta}{2} \frac{ \left(\frac{1}{n} \sum_{i \in [n]} \nabla F_i(x_{i, t-1}) \right)^2_j }{\sqrt{\left( \bar{B}_{t-t'} \right)^2_j + t' \epsilon^2}} 
+ \frac{L\eta^2}{2} \frac{ \left( \frac{1}{n}  \sum_{i \in [n]} \nabla F_i(x_{i, t-1}) \right)^2_j}{\left( \bar{B}_{t-t'} \right)^2_j + t' \epsilon^2} \\
&\leq -\frac{\eta}{2} \frac{ \left(\frac{1}{n} \sum_{i \in [n]} \nabla F_i(x_{i, t-1}) \right)^2_j }{\sqrt{\left( \bar{B}_{t-t'} \right)^2_j + t' \epsilon^2}} 
+ \frac{\eta}{2} \frac{ \left( \frac{1}{n}  \sum_{i \in [n]} \nabla F_i(x_{i, t-1}) \right)^2_j}{\left( \bar{B}_{t-t'} \right)^2_j + t' \epsilon^2} \\
&\leq 0.
\end{align*}

Using $\epsilon \geq p \rho$ and $p \leq 1$, we have
\begin{align*}
&\mcir{6} = \E\left[ \frac{ \left( \frac{1}{n}  \sum_{i \in [n]} (\nabla f_i(x_{i, t-1}) - \nabla F_i(x_{i, t-1})) \right)^2_j }{\left( \bar{B}_{t-t'} \right)^2_j + t' \epsilon^2} \right] \\
&= \frac{1}{n^2} E\left[ \frac{ \sum_{i \in [n]} \left( \nabla f_i(x_{i, t-1}) - \nabla F_i(x_{i, t-1}) \right)^2_j }{\left( \bar{B}_{t-t'} \right)^2_j + t' \epsilon^2} \right] \\
&\leq \frac{1}{n} E\left[ \frac{ \frac{1}{n} \sum_{i \in [n]} \left( \nabla f_i(x_{i, t-1}) \right)^2_j }{\left( \bar{B}_{t-t'} \right)^2_j + t' \epsilon^2} \right] \\
&\leq \frac{1}{n} E\left[ \frac{ \frac{1}{n} \sum_{i \in [n]} \left( \nabla f_i(x_{i, t-1}) \right)^2_j }{p^2 \left( \bar{B}_{t-t'} \right)^2_j + t' p^2\rho^2} \right] \\
&\leq \frac{1}{n p^2} E\left[ \frac{ \frac{1}{n} \sum_{i \in [n]} \left( \nabla f_i(x_{i, t-1}) \right)^2_j }{\left( \bar{B}_{t} \right)^2_j } \right],
\end{align*}
where $\left( \bar{B}_{t} \right)^2_j = b^2_0 + \sum_{s=1}^t \frac{1}{n} \sum_{i \in [n]} \left( G_{i, s} \right)^2_j$.

Finally, using smoothness, we have
\begin{align*}
&\sum_{j=1}^d\mcir{5} = \frac{\eta}{2} \sum_{j=1}^d \frac{ \left( \nabla F(\bar{x}_{t-1}) - \frac{1}{n} \sum_{i \in [n]} \nabla F_i(x_{i, t-1}) \right)^2_j }{\sqrt{\left( \bar{B}_{t-t'} \right)^2_j + t' \epsilon^2}} \\
&\leq \frac{\eta}{2} \sum_{j=1}^d \left( \frac{1}{n} \sum_{i \in [n]} \nabla F_i(\bar{x}_{t-1}) - \frac{1}{n} \sum_{i \in [n]} \nabla F_i(x_{i, t-1}) \right)^2_j \\
&\leq \frac{\eta}{2}\frac{1}{n} \sum_{i \in [n]} \| \nabla F_i(\bar{x}_{t-1}) - \nabla F_i(x_{i, t-1}) \|^2 \\
&\leq \frac{\eta L^2}{2n} \sum_{i \in [n]} \| \bar{x}_{t-1} - x_{i, t-1} \|^2.
\end{align*}

Note that $\bar{x}_{t-1}$ is synchronized across the workers. Thus, we have 
\begin{align*}
\bar{x}_{t-1} &= \bar{x}_{t-t'} - \eta \sum_{s=1}^{t'-1} \frac{1}{n} \sum_{i \in [n]} \frac{G_{i, t-t'+s}}{B^2_{i, t-t'} + s \epsilon^2 \1}, \\
x_{i, t-1} &= \bar{x}_{t-t'} - \eta \sum_{s=1}^{t'-1} \frac{G_{i, t-t'+s}}{B^2_{i, t-t'} + s \epsilon^2 \1}.
\end{align*}

Then, we have
\begin{align*}
&\sum_{j=1}^d\mcir{5} \\
&\leq \frac{\eta L^2}{2n} \sum_{j=1}^d \sum_{i \in [n]} \left( \bar{x}_{t-1} - x_{i, t-1} \right)^2_j \\
&\leq \frac{\eta^3 L^2}{2n} \sum_{j=1}^d \sum_{i \in [n]} \left[ \sum_{s=1}^{t'-1} \left( \frac{1}{n} \sum_{k \in [n]} \frac{G_{k, t-t'+s}}{\bar{B}^2_{t-t'} + s \epsilon^2 \1} - \frac{G_{i, t-t'+s}}{\bar{B}^2_{t-t'} + s \epsilon^2 \1} \right) \right]^2_j \\
&\leq \frac{2 \eta^3 L^2}{n} \sum_{j=1}^d \sum_{i \in [n]} \left( \sum_{s=1}^{t'-1} \frac{G_{i, t-t'+s}}{\bar{B}^2_{t-t'} + s \epsilon^2 \1} \right)^2_j \\
&\leq \frac{2 \eta^3 L^2 H}{n} \sum_{j=1}^d \sum_{i \in [n]} \sum_{s=1}^{t'-1} \frac{\left( G_{i, t-t'+s} \right)^2_j}{\left( \bar{B}^2_{t-t'} \right)^2_j + s \epsilon^2}  \\
&\leq \frac{2 \eta^3 L^2 H}{n p^2} \sum_{j=1}^d \sum_{i \in [n]} \sum_{s=1}^{H} \frac{\left( G_{i, t-t'+s} \right)^2_j}{\left( \bar{B}^2_{t-t'+s} \right)^2_j}.
\end{align*}

Now, we combine all the ingredients above:
\begin{align*}
&\E\left[ F(\bar{x}_t) - F(\bar{x}_{t-1}) \right] \\
&\leq \sum_{j=1}^d \E\left[ \mcir{1} \right] + \frac{L \eta^2}{2} \sum_{j=1}^d \E\left[ \mcir{2} \right] \\
&\leq \sum_{j=1}^d \E\left[ \mcir{3} + \mcir{4} + \mcir{5} \right] + \frac{L \eta^2}{2}\sum_{j=1}^d \E\left[ \mcir{6} + \mcir{7} \right] \\
&\leq -\frac{\eta}{2} \frac{ \| \nabla F(\bar{x}_{t-1}) \|^2 }{\sqrt{b^2_0 + T\frac{\epsilon^2}{p^2}}} 
+ \sum_{j=1}^d \E\left[ \frac{2 \eta^3 L^2 H}{n p^2} \sum_{i \in [n]} \sum_{s=1}^{H} \frac{\left( G_{i, t-t'+s} \right)^2_j}{\left( \bar{B}^2_{t-t'+s} \right)^2_j} \right] \\
&\quad + \sum_{j=1}^d \E\left[ \frac{L \eta^2}{2n p^2} \frac{ \frac{1}{n} \sum_{i \in [n]} \left( \nabla f_i(x_{i, t-1}) \right)^2_j }{\left( \bar{B}_{t} \right)^2_j } \right].
\end{align*}

By re-arranging the terms, we have
\begin{align*}
&\| \nabla F(\bar{x}_{t-1}) \|^2 \\
&\leq \frac{2 \sqrt{b^2_0 + T\frac{\epsilon^2}{p^2}} \E\left[ F(\bar{x}_{t-1}) - F(\bar{x}_{t}) \right]}{\eta} \\
&\quad + \frac{4 \eta^2 L^2 H \sqrt{b^2_0 + T\frac{\epsilon^2}{p^2}} }{p^2} \sum_{j=1}^d \E\left[ \sum_{s=1}^{H} \frac{\frac{1}{n} \sum_{i \in [n]} \left( G_{i, t-t'+s} \right)^2_j}{\left( \bar{B}^2_{t-t'+s} \right)^2_j} \right] \\
&\quad + \frac{L \eta \sqrt{b^2_0 + T\frac{\epsilon^2}{p^2}} }{n p^2} \sum_{j=1}^d \E\left[ \frac{ \frac{1}{n} \sum_{i \in [n]} \left( \nabla f_i(x_{i, t-1}) \right)^2_j }{\left( \bar{B}_{t} \right)^2_j } \right].
\end{align*}

By telescoping and taking total expectation, we have 
\begin{align*}
&\frac{\E \left[ \sum_{t=1}^T \| \nabla F(\bar{x}_{t-1}) \|^2 \right]}{T} \\
&\leq \frac{2 \sqrt{b^2_0 + T\frac{\epsilon^2}{p^2}} \E\left[ F(\bar{x}_{t0}) - F(\bar{x}_{T}) \right]}{\eta T} \\
&\quad + \frac{4 \eta^2 L^2 H \sqrt{b^2_0 + T\frac{\epsilon^2}{p^2}} }{T p^2} \sum_{j=1}^d \E\left[ \sum_{t=1}^T \sum_{s=1}^{H} \frac{\frac{1}{n} \sum_{i \in [n]} \left( G_{i, t-t'+s} \right)^2_j}{\left( \bar{B}^2_{t-t'+s} \right)^2_j} \right] \\
&\quad + \frac{L \eta \sqrt{b^2_0 + T\frac{\epsilon^2}{p^2}} }{n T p^2} \sum_{j=1}^d \E\left[ \sum_{t=1}^T \frac{ \frac{1}{n} \sum_{i \in [n]} \left( \nabla f_i(x_{i, t-1}) \right)^2_j }{\left( \bar{B}_{t} \right)^2_j } \right] \\
&\leq \frac{2 \sqrt{b^2_0 + T\frac{\epsilon^2}{p^2}} \E\left[ F(\bar{x}_{t0}) - F(\bar{x}_{T}) \right]}{\eta T} \\
&\quad + \frac{4 \eta^2 L^2 H^2 \sqrt{b^2_0 + T\frac{\epsilon^2}{p^2}} }{T p^2} \sum_{j=1}^d \E\left[ \sum_{t=1}^T \frac{ \frac{1}{n} \sum_{i \in [n]} \left( \nabla f_i(x_{i, t-1}) \right)^2_j }{\left( \bar{B}_{t} \right)^2_j } \right] \\
&\quad + \frac{L \eta \sqrt{b^2_0 + T\frac{\epsilon^2}{p^2}} }{n T p^2} \sum_{j=1}^d \E\left[ \sum_{t=1}^T \frac{ \frac{1}{n} \sum_{i \in [n]} \left( \nabla f_i(x_{i, t-1}) \right)^2_j }{\left( \bar{B}_{t} \right)^2_j } \right] \\
&\leq \frac{2 \sqrt{b^2_0 + T\frac{\epsilon^2}{p^2}} \E\left[ F(\bar{x}_{t0}) - F(\bar{x}_{T}) \right]}{\eta T} \\
&\quad + \frac{4 \eta^2 L^2 H^2 \sqrt{b^2_0 + T\frac{\epsilon^2}{p^2}} }{T p^2} d \log\left( b_0^2 + T \rho^2 \right) \\
&\quad + \frac{L \eta \sqrt{b^2_0 + T\frac{\epsilon^2}{p^2}} }{n T p^2} d \log\left( b_0^2 + T \rho^2 \right).
\end{align*}

\end{proof}

\end{document}